\documentclass[11pt]{article}

\usepackage{newpxtext}
\usepackage[euler-digits]{eulervm}

\usepackage{amsthm} 
\usepackage{amsfonts}

\usepackage[mathscr]{eucal}
\usepackage{bm}
\usepackage[english]{babel}% for Brittish English
\usepackage{microtype}		% for character protrusion and font expansion (only with pdflatex)
\usepackage{booktabs,ctable,multirow}
\usepackage{graphicx}
\usepackage{float}
\usepackage{array}
\usepackage{amsmath}
\usepackage{url} % not crucial - just used below for the URL 

\usepackage{xcolor}
\definecolor{violet}{RGB}{12,6,120}
\usepackage[bookmarks,bookmarksopen,bookmarksnumbered, colorlinks=true,citecolor=violet]{hyperref}

\usepackage{xspace}

\usepackage[square]{natbib}

\newcommand{\eqdef}{\stackrel{\text{\tiny def}}{=}}
\newtheorem{definition}{Definition}

\newtheorem{lemma}{Lemma}
\newtheorem{theorem}{Theorem}
\newtheorem{tproof}{Proof of Theorem}

\newtheorem{remark}{Remark}

\DeclareRobustCommand{\p}{\mspace{4mu}}
\DeclareRobustCommand{\pp}{\mspace{8mu}}
\DeclareRobustCommand{\PP}{\mspace{16mu}}
\DeclareRobustCommand{\mm}{\mspace{-8mu}}
\DeclareRobustCommand{\MM}{\mspace{-16mu}}
\DeclareRobustCommand{\MMm}{\mspace{-24mu}}
\DeclareRobustCommand{\MMM}{\mspace{-32mu}}

\newcommand{\ijp}{i^\prime,j^\prime}
\newcommand{\ijijp}{ij,i^\prime j^\prime}
\newcommand{\nEdgB}{\nedge{\bB}}
\newcommand{\cE}{\mathcal{E}}
\newcommand{\nE}{\overline{\mathcal{E}}}

\newcommand{\edgB}{\cE \mspace{-2mu} \left(\bB \right)}
\newcommand{\edgNB}{\nE \mspace{-2mu} \left(\bB \right)}

\DeclareRobustCommand{\nedge}[1]{\Big \lvert \cE \big (#1\big) \Big \rvert \mspace{-2mu}}
\DeclareRobustCommand{\EEbarre}{\stackrel{\scriptstyle [i,j] \in \edgNB}{\scriptstyle [\ijp] \in \edgB}}
\DeclareRobustCommand{\allij}{1 \le i < j \le n}
\DeclareRobustCommand{\iipjjp}{\stackrel{\scriptstyle \allij}{\scriptstyle 1 \le i^\prime < j^\prime \le n }}

\DeclareMathOperator{\theargmin}{argmin}

\DeclareMathOperator{\pr}{\mathbb{P}}
\DeclareMathOperator{\diag}{diag}

\newcommand{\eps}{\varepsilon}

\NewDocumentCommand \F {o m}   %m = mandatory in {}, o = optional in []
{
  \IfNoValueTF {#2} {\widehat{F}\left(#2\right)}{\widehat{F}_{#1}\mspace{-2mu}\left(#2\right) }
}

\DeclareRobustCommand{\o}[1]{{\scriptstyle\mathscr{O}} \textstyle \left(\displaystyle #1\right)}
\DeclareRobustCommand{\O}[1]{{\displaystyle \mathcal{O}} \mspace{-4mu} \left(#1\right)}

\DeclareRobustCommand{\op}[1]{{\scriptstyle \mathscr{O}}_P \textstyle \left(\displaystyle #1\right)}
\DeclareRobustCommand{\Op}[1]{{\displaystyle \mathscr{O}}_P \left(#1\right)}

\newcommand{\hA}{{\widehat{\bm{A}}}}
\newcommand{\bAp}{\bm{A}^\prime}
\newcommand{\bAk}{\bm{A}^{(k)}}

\newcommand{\bd}{\bm{d}}
\newcommand{\be}{\bm{e}}

\newcommand{\bu}{\bm{u}}

\newcommand{\bz}{\bm{z}}

\newcommand{\bA}{\bm{A}}
\newcommand{\bB}{\bm{B}}
\newcommand{\bC}{\bm{C}}
\newcommand{\bD}{\bm{D}}

\newcommand{\bI}{\bm{I}}
\newcommand{\bJ}{\bm{J}}
\newcommand{\bK}{\bm{K}}
\newcommand{\bL}{\bm{L}}
\newcommand{\bP}{\bm{P}}
\newcommand{\bQ}{\bm{Q}}
\newcommand{\bR}{\bm{R}}

\newcommand{\bmu}{\bm{\mu}}

\newcommand{\bPi}{\bm{\Pi}}
\newcommand{\bsg}{\bm{\sigma}}
\newcommand{\bRho}{\bm{\rho}}
\newcommand{\hRho}{{\widehat{\bRho}}}

\newcommand{\cG}{\mathcal{G}}

\newcommand{\cL}{\mathcal{L}}

\newcommand{\cS}{\mathcal{S}}

\newcommand{\Lpi}{\bm{\mathcal{L}}^\dagger}

\DeclareRobustCommand{\lbr}{\langle}
\DeclareRobustCommand{\rbr}{\rangle}

\DeclareRobustCommand{\argmin}[1]{\underset{#1}{\theargmin}\mspace{4mu}}
\DeclareRobustCommand{\er}[1]{{R_{#1}}}
\DeclareRobustCommand{\erp}[1]{{R^\prime_{#1}}}

\DeclareRobustCommand{\E}[1]{\mathbb{E} \left [#1 \right]}             % expectation
\DeclareRobustCommand{\lun}[1]{\big\| #1 \big\|_1}             % l1 norm

\DeclareRobustCommand{\fm}[1]{\bmu \mspace{-2mu} \big [\mspace{-2mu} #1 \mspace{-2mu}\big]}                 % frechet mean
\DeclareRobustCommand{\sfm}[1]{\widehat{\bm{\mu}} \mspace{-2mu}\big [\mspace{-2mu}#1 \mspace{-2mu}\big]}    % frechet mean 
\DeclareRobustCommand{\smd}[1]{\widehat{\bm{m}}\left [#1 \right]}                                           % median

\DeclareRobustCommand{\sE}[1]{\widehat{\mathbb{E}}\mspace{-2mu}\left[#1 \right]} % sample expectation

\DeclareRobustCommand{\prob}[1]{\mspace{2mu}\mathbb{P}\mspace{-2mu}\left(\mspace{-2mu} #1 \mspace{-2mu}\right)}

\newcommand{\aij}{a_{ij}}
\newcommand{\aijp}{a_{i^\prime j^\prime}}

\newcommand{\akij}{a^{(k)}_{ij}}
\newcommand{\akijp}{a^{(k)}_{i^\prime j^\prime}}

\newcommand{\roij}{\rho_{ij, i^\prime j^\prime}}
\newcommand{\pij}{p_{ij}}
\newcommand{\pijp}{p_{i^\prime j^\prime}}

\newcommand{\SBM}{\mathcal{G}\mspace{-1mu}\big(\mspace{-2mu}n,p,q\mspace{-1mu}\big)}
\newcommand{\Gk}{G^{(k)}}

\newcommand{\hP}{{\widehat{\bP}}}

\DeclareMathOperator{\rp}{rp}

\DeclareMathOperator{\drpt}{d_{\rp 2}}

\newcommand{\ER}{Erd\H{o}s-R\'enyi\xspace}
\newcommand{\fr}{Fr\'echet\xspace}

\newcommand{\citeg}[1]{\citep[e.g.,][]{#1}}

\begin{document}
\bibliographystyle{plainnat}

\title{When does the mean network capture\\ the topology of a sample of networks?}
\author{Fran\c{c}ois G. Meyer\\
  Applied Mathematics, University of Colorado at Boulder, Boulder CO 80305\\
  \href{mailto:fmeyer@colorado.edu}{\sf \small fmeyer@colorado.edu}\\{\small \url{https://francoismeyer.github.io}}
}
\date{}
\maketitle

\begin{abstract}

The notion of \fr mean (also known as ``barycenter'') network is the workhorse of most
machine learning algorithms that require the estimation of a ``location'' parameter to
analyse network-valued data. In this context, it is critical that the network barycenter
inherits the topological structure of the networks in the training dataset. The metric --
which measures the proximity between networks -- controls the structural properties of the
barycenter.

This work is significant because it provides for the first time analytical estimates of
the sample \fr mean for the stochastic blockmodel, which is at the cutting edge of
rigorous probabilistic analysis of random networks.  We show that the mean network
computed with the Hamming distance is unable to capture the topology of the networks in
the training sample, whereas the mean network computed using the effective resistance
distance recovers the correct partitions and associated edge density.

From a practical standpoint, our work informs the choice of metrics in the context where
the sample \fr mean network is used to characterize the topology of networks for
network-valued machine learning.\\

{\noindent  Keywords:} \fr mean; statistical network analysis\\
\end{abstract}
%%
%%
%%=========================================================================================================
%%
%%
\section{Introduction}
%%
%%
%%=========================================================================================================
%%
%%
 There has been recently a flurry of activity around the design of machine learning
algorithms that can analyze ``network-valued random variables'' \citep[e.g.,][and
  references
  therein]{dubey20,ghoshdastidar20,ginestet17,kolaczyk20,lunagomez21,petersen19,xu20,zambon19}. A
prominent question that is central to many such algorithms is the estimation of the mean
of a set of networks.  To characterize the mean network we borrow the notion of barycenter
from physics, and define the \fr mean as the network that minimizes the sum of the squared
distances to all the networks in the ensemble. This notion of centrality is well adapted
to metric spaces \citeg{chowdhury18,jain16b,kolaczyk20}, and the \fr mean network has
become a standard tool for the statistical analysis of network-valued data.

In practice, given a training set of networks, it is important that the topology of the
sample \fr mean captures the mean topology of the training set. To provide a theoretical
answer to this question, we estimate the mean network when the networks are sampled from a
stochastic block model. The stochastic block models \citep{snijders11,abbe18} have great
practical importance since they provide tractable models that capture the topology of real
networks that exhibit community structure. In addition, the theoretical properties (e.g.,
degree distribution, eigenvalues distributions, etc.) of this ensemble are well
understood. Finally, stochastic block models provide universal approximants to networks
and can be used as building blocks to analyse more complex networks
\citep{airoldi13,ferguson23a,olhede14}.

In this work, we derive the expression of the sample \fr mean of a stochastic block model
for two very different distances: the Hamming distance \citep{mitzenmacher17} and the
effective resistance perturbation distance \citep{monnig18}. The Hamming distance, which
counts the number of edges that need to be added or subtracted to align two networks
defined on the same vertex set, is very sensitive to fine scale fluctuations of the
network connectivity.  To detect larger scale changes in connectivity, we use the
resistance perturbation distance \citep{monnig18}. This network distance can be tuned to
quantify configurational changes that occur on a network at different scales: from the
local scale formed by the neighbors of each vertex, to the largest scale that quantifies
the connections between clusters, or communities \citep{monnig18}. See \citep[][and
  references therein]{Akoglu2015,donnat18,wills20c} for recent surveys on network
distances.

Our analysis shows that the sample \fr mean network computed with the Hamming distance is
unable to capture the topology of networks in the sample. In the case of a sparse
stochastic block model, the \fr mean network is always the empty network. Conversely, the
\fr mean computed using the effective resistance distance recovers the underlying network
topology associated with the generating process: the \fr mean discovers the correct
partitions and associated edge densities.
%%
%%
%%=========================================================================================================
%%
%%
\subsection{Relation to existing work}
%%
%%
%%=========================================================================================================
%%
%%
To the best of our knowledge, we are not aware of any theoretical derivation of the sample
\fr mean for any of the classic ensemble of random networks. Nevertheless, our work share
some strong connections with related research questions.\\
%%
%%
%%=========================================================================================================
%%
%%

{\bfseries The \fr mean network as a location parameter.} Several authors have proposed
simple models of probability measures defined on spaces of networks, which are
parameterized by a location and a scale parameter \citep{banks94,lunagomez21}. These
probability measures can be used to assign a likelihood to an observed network by
measuring the distance of that network to a central network, which characterizes the
location of the distribution. The authors in \citep{lunagomez21} explore two choices for
the distance: the Hamming distance, and a diffusion distance. Our choice of distances is
similar to that of \citep{lunagomez21}.
%%
%%
%%=========================================================================================================
%%
%%

{\bfseries Existing metrics for the \fr mean network.} The concept of \fr mean
necessitates a choice of metric (or distance) on the probability space of networks. The
metric will influence the characteristics that the mean will inherit from the network
ensemble. For instance, if the distance is primarily sensitive to large scale features
(e.g., community structure or the existence of highly connected ``hubs''), then the mean
will capture these large scale features, but may not faithfully reproduce the fine scale
connectivity (e.g., the degree of a vertex, or the presence of triangles).

One sometimes needs to compare networks of different sizes; the edit distance, which
allows for creation and removal of vertices, provides an elegant solution to this
problem. When the networks are defined on the same vertex set, the edit distance becomes
the Hamming distance \citep{han16}, which can also be interpreted as the entrywise
$\ell_1$ norm between the two adjacency matrices.  Replacing the $\ell_1$ norm with the
$\ell_2$ norm yields the Frobenius norm, which has also been used to compare networks
(modulo an unknown permutation of the vertices -- or equivalently by comparing the
respective classes in the quotient set induced by the action of the group of permutations
\citep{jain16b,kolaczyk20}). We note that the computation of the sample \fr mean network
using the Hamming distance is NP-hard~\citeg{chen19}. For this reason, several
alternatives have been proposed \citeg{ginestet17}.

Both the Hamming distance and Frobenius norm are very sensitive to the fine scale edge
connectivity. To probe a larger range of scales, one can compute the mean network using
the eigenvalues and eigenvectors of the respective network adjacency matrices
\citep{ferguson23a,ferrer05,white07}.
%%
%%
%%=========================================================================================================
%%
%%
\subsection{Content of the paper: our main contributions}
%%
%%
%%=========================================================================================================
%%
%%
Our contributions consists of two results.\\

{\noindent \bfseries The network distance is the Hamming distance.} We prove that when the
probability space is equipped with the Hamming distance, then the sample \fr mean network
converges in probability to the sample median network (computed using the majority rule), in
the limit of large sample size. This result has significant practical
consequences. Consider the case where one needs to estimate a ``central network'' that
captures the connectivity structure of a training set of sparse networks. Our work implies
that if one uses the Hamming distance, then the sample \fr mean will be the empty network.\\

{\noindent \bfseries The network distance is the resistance perturbation distance.} We
prove that when the probability space is equipped with the resistance perturbation
distance, then the adjacency matrix of the sample \fr mean converges to the sample mean
adjacency matrix with high probability, in the limit of large network size. Our
theoretical analysis is based on the stochastic block model \citep{abbe18}, a model of
random networks that exhibit community structure. In practical applications, our work
suggests that one should use the effective resistance distance to learn the mean topology
of a sample of networks.
%%
%%
%%=========================================================================================================
%%
%%
\subsection{Outline of the paper}
%%
%%
%%=========================================================================================================
%%
%%
In Section \ref{model}, we introduce the stochastic block model, the Hamming and
resistance distances that are defined on this probability space.  The reader who is
already familiar with the network models and distances can skip to Section \ref{results}
wherein we detail the main results, along with the proofs of the key results. In
section~\ref{discussion}, we discuss the implications of our work. The proofs of some
technical lemmata are left aside in section~\ref{lespreuves}.
%%
%%
%%=========================================================================================================
%%
%%
\section{Network ensemble and distances\label{model}}
%%
%%
%%=========================================================================================================
%%
%%
\subsection{The network ensemble}
%%
%%
%%=========================================================================================================
%%
%%
Let $\cG$ be the set of all simple labeled networks with vertex set $[n] \eqdef \big\{1, \ldots
,n\big\}$, and let $\cS$ be the set of $n \times n$ adjacency matrices of networks in
$\cG$,
\begin{equation}
  \cS = \left \{
  \bA \in \{0,1\}^{n \times n}; \text{where} \; a_{ij} = a_{ji},\text{and}  \; a_{i,i} = 0; \; 1 \leq i < j \leq n
  \right\}.
  \label{adjacency_matrices}
\end{equation}
Because there is a unique correspondence between a network $G=(V,E)$ and its adjacency
matrix $\bA$, we sometimes (by an abuse of the language) refer to an adjacency matrix
$\bA$ as a network. Also, without loss of generality we assume throughout the paper that
the network size $n$ is even.\\

We introduce the matrix $\bP$ that encodes the edges density within each community and 
across communities. $\bP$ can be written as the Kronecker product of the following two matrices,
\begin{equation}
  \bP = 
  \begin{bmatrix}
    p & q\\
    q & p
  \end{bmatrix}
  \otimes \bJ_{n/2}
  \label{les_probabilites_de_succes}
\end{equation}
where $\bJ_{n/2}$ is the $n/2 \times n/2$ matrix with all entries equal to $1$. We denote
by $\SBM$, the probability space $\cS$ equipped with the probability measure,
\begin{equation}
  \forall \bA \in \cS,\mspace{32mu}
  \prob{\bA} =
  \prod_{\stackrel{\scriptstyle 1 \le i \le n/2}{\scriptstyle  1 \le j \le n/2}}
  \mspace{-16mu} p^{a_{ij}} [1- p]^{1-a_{ij}}
  \mspace{-16mu} \prod_{\stackrel{\scriptstyle \mspace{24mu} 1 \le i \le n/2}{\scriptstyle  \mspace{-48mu} n/2 + 1 \le j \le n}}
  \mspace{-16mu} q^{a_{ij}} [1- q]^{1-a_{ij}}.
  \label{laproba}
\end{equation}
$\SBM$ is referred to as a two-community stochastic blockmodel \citep{abbe18}. One can
interpret the stochastic blockmodel as follows: the nodes of a network $G
\in \cG$ are partitioned into two communities. The first $n/2$ nodes constitute community
$C_1$; the second community, $C_2$, comprises the remaining $n/2$ nodes. Edges in the
graph are drawn from independent Bernoulli random variables with the following probability
of success: $p$ for edges within each community, and $q$ for the across-community edges.
%%
%%
%%=========================================================================================================
%%
%%
\subsection{The Hamming distance between networks}
%%
%%
%%=========================================================================================================
%%
%%
Let $\bA$ and $\bAp$ be the adjacency matrices of two unweighted networks defined on the
same vertex set. The Hamming distance \citep{mitzenmacher17} is defined as follows.
%%
%%
%%=========================================================================================================
%%
%%
\begin{definition}
  The Hamming distance between $\bA$ and $\bAp$ is defined as 
  \begin{equation}
    d_H(\bA,\bAp) = \frac{1}{2}\lun{\bA - \bA^\prime},
  \end{equation}
  \label{hamming}
  where the elementwise $\ell_1$ norm of a matrix $\bA$ is given by
  $\lun{\bA} = \sum_{1 \le  i,j \le n} |a_{ij}|$. 
\end{definition}
%%
%%
%%=========================================================================================================
%%
%%
Because the distance $d_H$ is not concerned about the locations of the edges that are
different between the two graphs, $d_H(\bA,\bAp)$ is oblivious to topological differences
between $\bA$ and $\bAp$. For instance, if $\bA$ and $\bAp$ are sampled from $\SBM$, then
the complete removal of the across-community edges induces the same distance as the
removal, or addition, of that same number of edges in either community. In other words, a
catastrophic change in the network topology cannot be distinguished from benign
fluctuations in the local connectivity within either community. To address the limitation
of the Hamming distance we introduce the resistance distance \citep{monnig18}.
%%
%%
%%=========================================================================================================
%%
%%
\subsection{The resistance (perturbation) distance between networks}
%%
%%
%%=========================================================================================================
%%
%%
For the sake of completeness, we review the concept of effective resistance
\citep[e.g.,][]{doyle84,klein1993}. Given a weighted network $G=(V,E)$, we transform $G$
into a resistor network by replacing each edge $e$, with weight $w_e$, by a resistor with
resistance $1/w_e$.  The {\em effective resistance} $R_{ij}$ between two vertices $i$ and
$j$ is defined as the voltage applied between $i$ and $j$ that is required to maintain a
unit current through the terminals formed by $i$ and $j$. The effective resistance
$R_{ij}$ provides a measure of vertex affinity in terms of connectivity.

The resistance-perturbation distance (or resistance distance for short) is based on
comparing the effective resistances matrices $\bR$ and $\bR'$ of $G$ and $G'$
respectively. To simplify the discussion, we only consider networks that are connected with
high probability. All the results can be extended to disconnected networks as explained in
\citep{monnig18}.
%========================================================================================
\begin{definition}
\label{laresistance_def}
  Let $G=(V,E)$ and $G^\prime=(V,E^\prime)$ be two networks defined on the same vertex set
  $[n]$. Let $R$ and $R^\prime$ denote the effective resistances of  $G$ and $G^\prime$
  respectively. We define the resistance-perturbation distance \citep{monnig18} to be 
  \begin{equation}
    \drpt \left(G,G^\prime \right) = \sum_{1\leq i < j\leq n} \left \lvert \er{ij} - \erp{ij} \right\rvert^2.
    \label{laresistance}
  \end{equation}
\end{definition}
%=======================================================================
\section{Main Results
  \label{results}}
%%
%%
%%=========================================================================================================
%%
%%
We first review the concept of sample \fr mean, and then present the main results.
%%
%%
%%=========================================================================================================
%%
%%
We consider the probability space $(\cS,\pr)$ formed by the adjacency matrices of networks
sampled from $\SBM$. We equip $\cS$ with a distance $d$, which is either the Hamming
distance or the resistance distance. Let $\bAk, 1\leq k \leq N$ be adjacency matrices
sampled independently from $\SBM$.
%%
%%
%%=========================================================================================================
%%
%%
\subsection{The sample \fr mean}
%%
%%
%%=========================================================================================================
%%
%%
The sample \fr function evaluated at $\bB \in \cS$ is defined by
\begin{equation}
  \F[2]{\bB} = \frac{1}{N}\sum_{k=1}^N d^2(\bB,\bAk).
  \label{frechet_function}
\end{equation}
The minimization of the \fr function $\F[2]{\bB}$ gives rise to the concept of sample \fr
mean \citep{frechet47}, or network barycenter \citep{sturm03}.
%%
%%
%%=========================================================================================================
%%
%%
\begin{definition}
  The sample \fr mean network is the set of adjacency matrices $\sfm{\pr}$ solutions to
  \begin{equation}
    \sfm{\pr} =  \argmin{\bB\in \cS} \frac{1}{N}\sum_{k=1}^N d^2(\bB,\bAk).
    \label{sample-frechet-mean}
  \end{equation}
\end{definition}
%%
%%
%%=========================================================================================================
%%
%%
Solutions to the minimization problem (\ref{sample-frechet-mean}) always exist, but need
not be unique. In theorem~\ref{theorem1} and theorem~\ref{theorem2}, we prove that the
sample \fr mean network of $\SBM$ is unique, when $d$ is either the Hamming distance or
the resistance distance. \\

A word on notations is in order here. It is customary to denote by $\fm{\pr}$ the {\em
  population} \fr mean network of the probability distribution $\pr$, \citep[see
  e.g.,][]{sturm03}, since the adjacency matrix $\fm{\pr}$ characterizes the location of
the probability distribution $\pr$. Because we use hats to denote sample (empirical)
estimates, we denote by $\sfm{\pr}$ the adjacency matrix of the {\em sample} \fr mean
network.
%%
%%
%%=========================================================================================================
%%
%%
\subsection{The sample \fr mean of $\SBM$ computed with the Hamming distance}
%%
%%
%%=========================================================================================================
%%
%%
The following theorem shows that the sample \fr mean network converges in probability to the
sample \fr median network, computed using the majority rule, in the limit of large sample
size, $N$.
%%
%%
%%=========================================================================================================
%%
%%
\begin{theorem}
  \label{theorem1}
Let $\sfm{\pr}$ be the sample \fr mean network computed using the Hamming distance. Then,
  \begin{equation}
    \forall \eps >0, \mspace{16mu} \exists N_0, \forall N \ge N_0, \mspace{16mu} \prob{d_H(\sfm{\pr}, \smd{\pr}) < \eps} \ge 1 - \eps.
  \end{equation}
  where $\smd{\pr}$ is the adjacency matrix computed using the majority rule,
  \begin{equation}
    \forall i,j \in [n],\quad   \smd{\pr}_{ij} =
    \begin{cases}
      1 & \text{if} \; \sum_{k=1}^N \akij \ge N/2,\\
      0 &  \text{otherwise.}
    \end{cases}
    \label{majority-rule}
  \end{equation}
\end{theorem}
%%
%%
%%=========================================================================================================
%%
%%
\begin{remark}
The matrix $\smd{\pr}$ is the sample \fr median network (e.g., \citep{jiang01}, solution
to the following minimization problem \citep{banks94},
\begin{equation}
  \smd{\pr} =  \argmin{\bB \in \cS} \F[1]{\bB},
  \label{sample-frechet-median}
\end{equation}
where $\widehat{F}_1$ is the \fr function associated to the sample \fr median, defined by
  \begin{equation}
    \F[1]{\bB}  =  \frac{1}{N} \sum_{k=1}^N d_H(\bAk, \bB).
  \label{frechet-function-median}
  \end{equation}
\end{remark}
%%
%%
%%=========================================================================================================
%%
%%
Before deriving the proof of theorem~\ref{theorem1}, we introduce an extension of the
Hamming distance to weighted networks. We remember that the sample \fr mean network
computed using the Hamming distance has to be an unweighted network, since the Hamming
distance is only defined for unweighted networks. This theoretical observation
notwithstanding, the proof of theorem~\ref{theorem1} becomes much simpler if we introduce
an extension of the Hamming distance to weighted networks; in truth, we extend a slightly
different formulation of the Hamming distance.\\

Let $\bA,\bB \in \cS$ be two unweighted adjacency matrices. Because $d_H(\bA,\bB)$ counts
the number of (unweighted) edges that are different between the graphs, we have
  \begin{equation}
    d_H(\bA,\bB) =
    \sum_{1 \le i< j \le n} \MM a_{ij} +
    \sum_{1 \le i< j \le n} \MM b_{ij}
    -2 \MM \sum_{1\leq i < j \leq n}\MM a_{ij} b_{ij}. \label{hamming2}
  \end{equation}
Now, assume that $\bA$ and $\bB$ are two weighted adjacency matrices, with $a_{ij},b_{ij}
\in [0,1]$. A natural extension of (\ref{hamming2}) to matrices with entries in $[0,1]$ is
therefore given by
  \begin{equation}
    \delta(\bA,\bB) =
    \sum_{1 \le i< j \le n} \MM a_{ij} +
    \sum_{1 \le i< j \le n} \MM b_{ij}
    -2 \MM \sum_{1\leq i < j \leq n}\MM a_{ij} b_{ij}. 
  \end{equation}
The function $\delta$, defined on the space of weighted adjacency matrices with weights in
$[0,1]$, satisfies all the properties of a distance, except for the triangle inequality.\\
%%
%%
%%=========================================================================================================
%%
%%

We now introduce the sample probability matrix $\hP$, estimate of $\bP$ (see
(\ref{les_probabilites_de_succes})) and the sample correlation $\hRho$. Let $\bAk, 1\leq
k \leq N$ be adjacency matrices sampled independently from $\SBM$. We define
  \begin{equation}
    \hP_{ij} \eqdef \sE{\aij} \eqdef \frac{1}{N} \sum_{k=1}^N \akij.
    \label{sample_probability}
  \end{equation}
and 
  \begin{equation}
    \hRho_{\ijijp} \eqdef \sE{\roij} \eqdef \frac{1}{N} \sum_{k=1}^N \akij \akijp.
    \label{sample_correlation}
  \end{equation}
%%
%%
%%=========================================================================================================
%%
%%
We can combine the definitions of $\delta$ and $\hP$ to derive the following expression
for the \fr function $\widehat{F}_1$ for the sample median, defined by (\ref{frechet-function-median}),
\begin{equation}
  \F[1]{\bB} = \delta(\bB,\hP). \label{frechet-median-delta}
\end{equation}
The proof for this identity is very similar to the proof of the next lemma, and is
omitted for brevity.\\
%%
%%
%%=========================================================================================================
%%
%%

We are now ready to present the proof of theorem~\ref{theorem1}. 
\begin{tproof}
The proof relies on the observation (formalized in lemma~\ref{lemma4}) that the \fr
function $\F[2]{\bB}$ can be expressed as the sum of a dominant term and a residual.  The
residual becomes increasingly small in the limit of large sample size (see
lemma~\ref{epsilon-concentrates}) and can be neglected. We show in
lemma~\ref{le_minimum_de_F} that the dominant term is minimum for the sample \fr median
network $\smd{\pr}$ (defined by (\ref{majority-rule})). We start
with the decomposition of $\F[2]{\bB}$ in terms of a dominant term and a residual.
%%
%%
%%=========================================================================================================
%%
%%

\begin{lemma}
  \label{lemma4}
  Let $\bB \in \cS$. We denote by $\edgB$ the set of edges of the network with adjacency
  matrix $\bB$, we denote by $\edgNB$ the set of ``nonedges''. Then
  \begin{equation}
    \F[2]{\bB}  = \delta^2 (\bB,\hP)  - \MMm \sum_{\iipjjp} \MM \big( \hP_{ij} \hP_{\ijp} - \sE{\roij}  \big) 
    + \pp 4 \MM \sum_{\EEbarre}  \MMm \big( \hP_{ij} \hP_{\ijp} - \sE{\roij} \big). \label{frechet-sample-delta}
  \end{equation}
\end{lemma}
where $\hP$ is defined by (\ref{sample_probability}), and $\hRho$ is defined by (\ref{sample_correlation}).
%%
%%
%%=========================================================================================================
%%
%%
\begin{proof}
The proof of lemma~\ref{lemma4} is provided in section~\ref{lespreuves}.
\end{proof}

To call attention to the distinct roles played by the terms in (\ref{frechet-sample-delta}), we define 
the dominant term of $\F[2]{\bB}$,
  \begin{equation}
    \F[]{\bB} \eqdef \delta^2 (\bB,\hP)
    - \MM \sum_{\iipjjp} \MM \Big [\hP_{ij} \hP_{\ijp} - \sE{\roij} \Big], \label{la_fonction_F}
  \end{equation}
and the residual term $\zeta_N$ is defined by
  \begin{equation}
    \zeta_N(\bB) = 4 \MMm\sum_{\EEbarre} \MM \Big( \hP_{ij}\hP_{\ijp} - \sE{\roij}\Big), \label{le_reste}
  \end{equation}
  so that
\begin{equation}
  \F[2]{\bB} = \F[]{\bB} + \zeta_N(\bB).
\end{equation}
%%
%%
%%=========================================================================================================
%%
%%

The next step of the proof of theorem~\ref{theorem1} involves showing that the sample
median network, $\smd{\pr}$, (see (\ref{majority-rule})) , which is the minimizer of
$\F[1]{\bB}$ (see (\ref{frechet-function-median})), is also the minimizer of $\F[]{\bB}$.
\begin{lemma}
  \label{le_minimum_de_F}
  $\smd{\pr}$ satisfies: $\forall \bB \in \cS, \mspace{8mu} \F[\mspace{2mu}]{\smd{\pr}} \le \F[]{\bB}$.
\end{lemma}
%%
%%
%%=========================================================================================================
%%
%%
\begin{proof}[Proof of lemma~\ref{le_minimum_de_F}]
  We have 
  \begin{align}
    \F[]{\bB} & = \delta^2 (\bB,\hP)  - \MMm \sum_{\iipjjp} \MM \big( \hP_{ij} \hP_{\ijp} - \sE{\roij}  \big) 
  \end{align}
  Because $\smd{\pr}$ is the minimizer of $\F[1]{\bB} = \delta(\bB,\hP)$ (see
  (\ref{frechet-median-delta})), $\smd{\pr}$ is also the minimizer of $\delta^2
  (\bB,\hP)$. Finally, since $\sum_{\iipjjp} \mm \big( \hP_{ij} \hP_{\ijp} -
  \sE{\roij} \big)$ does not depend on $\bB$, $\smd{\pr}$ is the minimizer of $\F[]{\bB}$.
\qed
\end{proof}
%%
%%
%%=========================================================================================================
%%
%%
\noindent We now turn our attention to the residual term, and we confirm in the next lemma that
$\zeta_N(\bB) = \Op{\frac{1}{\sqrt{N}}}$; to wit $\zeta_N(\bB)\sqrt{N}$ is bounded with
high probability.
%%
%%
%%=========================================================================================================
%%
%%
\begin{lemma}
  \label{epsilon-concentrates}
  $\forall \eps >0, \exists c>0, \forall \; N\ge 1$,
  \begin{equation}
    \prob{\bAk \sim \SBM; \big \lvert \zeta_N(\bB) \big \rvert < \frac{c}{\sqrt{N}}} > 1 - \eps.
    \label{zeta_va_a_zero}
  \end{equation}
\end{lemma}
\begin{proof}
The proof of lemma~\ref{epsilon-concentrates} is provided in section~\ref{proof-epsilon-concentrates}.
\end{proof}
%%
%%
%%=========================================================================================================
%%
%%  
\noindent The last technical lemma that is needed to complete the proof of theorem~\ref{theorem1} is
a variance inequality \citep{sturm03} for $\widehat{F}$. We assume that the entries of
$\bP$ are uniformly away from $1/2$ (this technical condition on $\bP$ prevents the
instability that occurs when estimating $\sfm{\pr}$ for $\pij = 1/2$).
\begin{lemma}
  \label{variance_inequality}
  We assume that there exists $\eta >0$ such that $\allij, \lvert \pij - 1/2 \rvert > \eta$. Then,
  $\exists \mspace{2mu} \alpha >0$
  \begin{equation}
    \forall \bB \in \cS, \quad \alpha \lun{\bB - \smd{\pr}}^2 \le \Big\lvert\F[]{\bB}
    - \F[]{\smd{\pr}}\Big\rvert ,
    \label{inegalite_variance}
  \end{equation}
  with high probability.
\end{lemma}
\begin{proof}
The proof of lemma~\ref{variance_inequality} is provided in section~\ref{proof_variance_inequality}.
\end{proof}

\noindent We are now in position to combine the lemmata and complete the proof of
theorem~\ref{theorem1}.\\

Let $\sfm{\pr}$ be the sample \fr mean network, and let $\smd{\pr}$ be the sample \fr
median network. By definition, $\sfm{\pr}$ is the minimizer of $\widehat{F}_2$, and thus
\begin{equation}
  \F[]{\sfm{\pr}} = \F[2]{\sfm{\pr}} - \zeta_N(\sfm{\pr}) \le \F[2]{\smd{\pr}} - \zeta_N(\sfm{\pr}) 
\end{equation}
Now, by definition of $\widehat{F}$ in (\ref{la_fonction_F}), we have
\begin{equation}
  \F[2]{\smd{\pr}} - \zeta_N(\sfm{\pr}) = \F[]{\smd{\pr}} + \zeta_N(\smd{\pr}) -\zeta_N(\sfm{\pr}),
\end{equation}
and therefore,
\begin{equation}
  0 \le \F[]{\sfm{\pr}} - \F[]{\smd{\pr}} \le \zeta_N(\smd{\pr}) -\zeta_N(\sfm{\pr}).
  \label{Lipschitz1}
\end{equation}
This last inequality, combined with (\ref{zeta_va_a_zero}) proves that $\F[]{\sfm{\pr}}$
converges to $\F[]{\smd{\pr}}$ for large $N$. We can say more; using the variance inequality
(\ref{inegalite_variance}), we prove that $\sfm{\pr}$ converges to $\smd{\pr}$ for large
$N$. Let~$\eps > 0$, from lemma \ref{variance_inequality}, there exists $\alpha > 0$ such
that
\begin{equation}
  \prob{\bAk \sim \SBM; \quad \alpha {\lun{\sfm{\pr} - \smd{\pr}}}^2 \le \Big \lvert \F[]{\sfm{\pr}} - \F[]{\smd{\pr}} \Big
    \rvert} > 1-\eps.
  \label{Lipschitz2}
\end{equation}
The term $\zeta_N(\smd{\pr}) -\zeta_N(\sfm{\pr})$ is controlled using
lemma~\ref{epsilon-concentrates},
\begin{equation}
  \exists C, \forall N \ge 1, 
  \prob{
    \forall \pr \in \cS,
    \Big \lvert \zeta_N(\smd{\pr}) -\zeta_N(\sfm{\pr}) \Big \rvert < \frac{C}{\sqrt{N}}
  } \ge 1 -\eps
  \label{residue}
\end{equation}
Combining (\ref{Lipschitz1}), (\ref{Lipschitz2}), and (\ref{residue}) we get
\begin{equation}
  \forall N \ge 1, \quad \prob{ {\lun{\sfm{\pr} - \smd{\pr} }}^2 < \frac{C}{\alpha\sqrt{N}}} > 1 - \eps.
\end{equation}
We conclude that $\exists N_1$ such that
\begin{equation}
  \forall N \ge N_1,  \quad \prob{\lun{\sfm{\pr} - \smd{\pr}} < \eps} > 1 - \eps,
\end{equation}
which completes the proof of the theorem. \qed
\end{tproof}
%%
%%
%%=========================================================================================================
%%
%%
\subsection{The sample \fr mean of $\SBM$ computed with the resistance distance}
%%
%%
%%=========================================================================================================
%%
%%
Here we equip the probability space $\big(\cS, \pr\big)$ with the resistance metric
defined by (\ref{laresistance}). Let $\bAk, 1\leq k \leq N$ be adjacency matrices 
sampled independently from $\SBM$, and let $\bR^{(k)}$ be their effective
resistances. Because the resistance metric relies on the comparison of connectivity at
multiple scales, we expect that the sample \fr mean network recovers the topology induced
by the communities.\\

The next theorem proves that the sample \fr mean converges toward the expected adjacency
matrix $\bP$ (see \ref{les_probabilites_de_succes}) in the limit of large networks.
\begin{theorem}
  \label{theorem2}
Let $\sfm{\pr}$ be the sample \fr mean computed using the effective resistance
distance.  Then
  \begin{equation}
    \sfm{\pr} = \E{\bA} = \bP,
    \label{FM-resistance}
  \end{equation}
  in the limit of large network size, with high probability.
\end{theorem}
%%
%%
%%=========================================================================================================
%%
%%
\begin{proof}[Proof of theorem~\ref{theorem2}]
The proof combines three elements. We first observe that the effective resistance of the
sample \fr mean is the sample mean effective resistance.
%%
%%
%%=========================================================================================================
%%
%%
\begin{lemma}
\label{resistance_de_la_moyenne}
Let $\sfm{\pr}$ be the sample \fr mean computed using the resistance distance. Then
\begin{equation}
  \widehat{R}_{ij}\eqdef  R_{ij}(\sfm{\pr}) = \frac{1}{N}\sum_{k=1}^N  R^{(k)}_{ij}
  \label{meanresistance}
\end{equation}
\end{lemma}
%%
%%
%%=========================================================================================================
%%
%%
\begin{proof}[Proof of lemma~\ref{resistance_de_la_moyenne}]
The proof relies on the observation that the \fr function in (\ref{sample-frechet-mean}),
is a quadratic function of $\widehat{R}_{ij} = R_{ij}(\sfm{\pr})$. Indeed, we have
\begin{equation}
  \frac{1}{N}\sum_{k=1}^N \sum_{1\leq i<j \leq n} \left \lvert \widehat{R}_{ij} - R^{(k)}_{ij} \right\rvert^2 =
  \sum_{1\leq i<j\leq n} \frac{1}{N}\sum_{k=1}^N  \left \lvert \widehat{R}_{ij} -
  R^{(k)}_{ij} \right\rvert^2
  \label{quadratic}
\end{equation}
where we have used the definition of the effective resistance distance given by
(\ref{laresistance}). The minimum of (\ref{quadratic}) is given by (\ref{meanresistance}).
\qed
\end{proof}
The previous lemma allows us to recover $\sfm{\pr}$ from its effective resistance, given
by (\ref{meanresistance}).  Instead of pursuing this strategy, we prove in the next lemma
that $1/N\sum_{k=1}^N R^{(k)}_{ij}$ is concentrated around $\E{R}_{ij}$ in the limit of
large network size, $n$. In addition, we derive the expression of $\E{R}_{ij}$ in the case
of a two community stochastic block model.
%%
%%
%%=========================================================================================================
%%
%%
\begin{lemma}
  \label{lemma_SBMresistanceConcentrates}
The effective resistance of the sample \fr  mean network $\sfm{\pr}$ is given by
\begin{equation}
  R_{ij}(\sfm{\pr}) = \E{R}_{ij} + \op{\frac{1}{n}},
  \label{SBMresistanceConcentrates}
\end{equation}
where $\E{\bR}$ is the population mean effective resistance of the probability
space $\big(\cS, \pr\big)$, and is given by
\begin{equation}
\E{\bR} =  \frac{4}{n(p+q)} \bJ + \frac{(p-q)}{(p+q)} \frac{4}{n^2q} \bK,
\label{expectedResistance}
\end{equation}
where $\bJ = \bJ_n$, and $\bK$ is the $n\times n$ matrix associated with the cross-community edges, 
\begin{equation}
  \bK = 
  \begin{bmatrix}
    0 & 1\\
    1 & 0
  \end{bmatrix}
  \otimes \bJ_{n/2}.
  \label{lamatriceK}
\end{equation}
\end{lemma}
%%
%%
%%=========================================================================================================
%%
%%
\begin{remark}
We can explain expression (\ref{expectedResistance}) with a simple circuit argument.  We
first analyse the case where $i$ and $j$ belong to the same community, say $C_1$. In this
case, we can neglect the other community $C_2$ because of the bottleneck created by the
across-community edges. Consequently, $C_1$ is approximately an \ER network wherein the
effective resistance $\er{ij}$ concentrates around $4/(n(p+q))$ \citep{luxburg14}, and we
obtain the first term in (\ref{expectedResistance}).\\

\noindent On the other hand, when the vertices $i$ and $j$ are in distinct communities, then a
simple circuit argument shows that
\begin{equation}
  \er{ij} \approx \frac{2}{n(p+q)} + \frac{1}{k} + \frac{2}{n(p+q)},
\end{equation}
where $k$ is the number of across-community edges, creating a bottleneck with effective
resistance $1/k$ between the two communities \citep{levin09}; each term
$2/(n(p+q))$ accounts for the effective resistance from node $i$ (respectively $j$) to a
node incident to an across-community edge. Because the number of across-community edges,
$k$, is a binomial random variable, it concentrates around its mean, $q n^2/4$. Finally,
$1/k$ is a binomial reciprocal whose mean is given by $4/(qn^2) + \o{1/n^3}$
\citep{rempala04}, and we recover the second term of (\ref{expectedResistance}).
\end{remark}
\begin{remark}
The result in lemma~\ref{lemma_SBMresistanceConcentrates} can be extended to a stochastic
block model of any geometry for which we can derive the analytic expression of the
dominant eigenvalues; see \citep[see e.g.,][]{lowe24,avrachenkov15} for equal size
communities, and \citep[see e.g.,][]{chakrabarty20} for the more general case of
inhomogeneous random networks.
\end{remark}
%%
%%
%%=========================================================================================================
%%
%%
Before presenting the proof of lemma~\ref{lemma_SBMresistanceConcentrates}, we introduce
the normalized Laplacian matrix \citep[e.g.][]{bapat10}. Let $\bA$ be the adjacency matrix
of a network $(V,E)$, and let $\bD$ be the diagonal matrix of degrees, $d_i = \sum_{j=1}^n
a_{ij}$. We normalize $\bA$ in a symmetric manner, and we define
\begin{equation}
  \hA = \bD^{-1/2} \bA \bD^{-1/2},
\end{equation}
where $\bD^{-1/2}$ is the diagonal matrix with entries $1/\sqrt{d_{i}}$. The normalized
Laplacian matrix is defined by
\begin{equation}
  \cL = \bI -\hA,
\end{equation}
where $\bI$ is the identity matrix. $\cL$ is positive semi-definite \citep{bapat10}, and we
will consider its Moore-Penrose pseudoinverse, $\Lpi$.
\begin{proof}[Proof of lemma~\ref{lemma_SBMresistanceConcentrates}]
The lemma relies on the characterization of $\bR$ in terms of $\Lpi$ \citep{bapat10},
  \begin{equation}
    \er{ij} = \lbr \bu_i - \bu_j, \Lpi (\bu_i - \bu_j)\rbr,
    \label{reffLdag}
  \end{equation}
  where $\bu_i = (1/\sqrt{d_i}) \mspace{8mu} \be_i$, and $\be_i$ is the $i^\text{th}$ vector of
  the canonical basis.

  Let $1=\lambda_1 \ge \lambda_2\ge \ldots\lambda_n \ge -1$ be the eigenvalues of $\hA$,
  and let $\bPi_1,\ldots, \bPi_n$ be the corresponding orthogonal projectors,
  \begin{equation}
    \hA = \sum_{m=1}^n \lambda_m \bPi_m,
  \end{equation}
  where $\bPi_1 =   \tau^{-1} {\bd^{1/2}}{\bd^{1/2}}^T$, with ${\bd^{1/2}}= \begin{bmatrix} \sqrt{d_1} & \cdots&
    \sqrt{d_n}
    \end{bmatrix}^T$, and $\tau = \sum_{i=1}^n d_i$. Because $\bPi_1$ is also the
  orthogonal projection on the null space of $\cL$, we have 
  \begin{equation}
    \Lpi  = (\cL + \bPi_1)^{-1} \mspace{-2mu} - \bPi_1
    = \left(\bI - (\hA - \bPi_1)\right)^{-1}  \mspace{-8mu} -\bPi_1
      =  \bI - \bPi_1 + \frac{\lambda_2}{1 - \lambda_2} \bPi_2 + \bQ,
    \label{LdaggProj}
  \end{equation}
where 
  \begin{equation}
    \bQ = \sum_{m=3}^n \frac{\lambda_m}{1 - \lambda_m} \bPi_m.
  \end{equation}
  Substituting (\ref{LdaggProj}) into (\ref{reffLdag}), we get
  \begin{equation}
    R_{ij} =  \lbr \bu_i - \bu_j, (\bI -\bPi_1) (\bu_i-\bu_j)\rbr
    +   \frac{\lambda_2}{1 - \lambda_2} \lbr \bu_i - \bu_j,  \bPi_2 (\bu_i-\bu_j)\rbr
    +  \lbr \bu_i - \bu_j,  \bQ(\bu_i-\bu_j)\rbr
    \label{Rij_proj}
  \end{equation}
  The first (and dominant) term of (\ref{Rij_proj}) is 
  \begin{equation}
    \lbr \bu_i - \bu_j, \left(\bI - \bPi_1\right) (\bu_i - \bu_j)\rbr  = 
    \lbr \bu_i - \bu_j, \bu_i - \bu_j\rbr  = \frac{1}{d_i} + \frac{1}{d_j}.
    \label{order1}
  \end{equation}
  Let us examine the second term of (\ref{Rij_proj}).  \cite{lowe24} provide the
    following estimate for $\lambda_2$,
  \begin{equation}
    \lambda_2 = \frac{p -q}{p+q} + \omega(n), \label{lavaleurpropre}
    \mspace{24mu} \text{where}     \mspace{24mu}
    \omega(n) =  \O{\sqrt{\frac{2\log{n}}{n(p+q)}}}.
  \end{equation}
  The corresponding eigenvector $\bz$ is given, with probability $(1 - \o{1})$, by \citep{deng21},
  \begin{equation}
    z_{i} =  \sigma_i \frac{1}{\sqrt{n}} + \o{\frac{1}{\sqrt{n}}}, \label{deuxiemevecteur}
  \end{equation}
    where the ``sign'' vector  $\bsg$, which encodes the community, is  given by
    \begin{equation}
      \sigma_i =
    \begin{cases}
      \mspace{12mu} 1 & \text{if} \mspace{56mu} 1 \le  i \le n/2,\\
      -1 & \text{if} \mspace{8mu} n/2 + 1 \le  i \le n.\\
    \end{cases}
  \end{equation}
    We derive from (\ref{deuxiemevecteur}) the following
  approximation to $\lbr \bu_i,  \bPi_2 \bu_j\rbr$,
  \begin{equation}
    \lbr \bu_i,  \bPi_2 \bu_j\rbr = \bu_i^T \bz\bz^T\bu_j = \frac{1}{\sqrt{d_i d_j}} z_i z_j =
    \frac{1}{n\sqrt{d_i d_j}} \sigma_i \sigma_j + \op{\frac{1}{n}}.
  \end{equation}
We therefore have
  \begin{equation}
    \lbr \bu_i - \bu_j ,  \bPi_2 (\bu_i - \bu_j)\rbr = \frac{1}{nd_i}  +
    \frac{1}{nd_j} -\frac{2}{n\sqrt{d_id_j}}\sigma_i\sigma_j + \op{\frac{1}{n}}.
  \end{equation}
The degree $d_i$ of node $i$ is a binomial random variable, which concentrates around its
mean, $p(n/2 -1) + qn/2 \approx n(p+q)/2$ for large network size. Also, $1/d_i$ is a
binomial reciprocal that also concentrates around its mean, which is given by $2/((p+q)n)
+ \o{1/n^2}$ \citep{rempala04}. We conclude that in the limit of large network size,
  \begin{equation}
    \lbr \bu_i - \bu_j ,  \bPi_2 (\bu_i - \bu_j)\rbr = \frac{4}{n^2(p+q)} \bigg( 1
    -\sigma_i\sigma_j\bigg) + \op{\frac{1}{n}}. \label{laprojection}
  \end{equation}
  Combining (\ref{lavaleurpropre}) and (\ref{laprojection}) yields
  \begin{equation}
    \lbr \bu_i - \bu_j , \frac{\lambda_2}{1-\lambda_2} \bPi_2 (\bu_i - \bu_j)\rbr
     = \frac{(p-q)}{(p+q)} \frac{4}{n^2q}
     \frac{ \big( 1 -\sigma_i\sigma_j\big)}{2} + \op{\frac{1}{n}}.
  \end{equation}
  We note that
  \begin{equation}
  \frac{(p-q)}{(p+q)} \frac{4}{n^2q}
     \frac{ \big( 1 -\sigma_i\sigma_j\big)}{2} =     
\begin{cases}
  \displaystyle \frac{4}{n(p+q)} & \text{if $i$ and $j$ are in the same community}\\
  \displaystyle \frac{4}{n(p+q)} + \frac{(p-q)}{(p+q)} \frac{4}{n^2q} & \text{if $i$ and $j$ are in
      different communities,}
\end{cases}
  \end{equation}
  which confirms that $\lbr \bu_i - \bu_j , \frac{\lambda_2}{1-\lambda_2} \bPi_2 (\bu_i -
  \bu_j)\rbr$ provides the correction in (\ref{expectedResistance}) created by the
  bottleneck between the communities. Finally, we show in
  section~\ref{preuve_tailleresidu} that the last term in the expansion (\ref{Rij_proj})
  of $R_{ij}$ can be neglected,
  \begin{equation}
    \big \lvert
    \lbr \bu_i - \bu_j , \bQ (\bu_i - \bu_j)\rbr    
    \big \rvert
    \leq \left (\frac{1}{d_i} + \frac{1}{d_j}\right) \frac{8\sqrt{2}}{(np)^{3/2}}  \quad \text{almost surely.}
    \label{tailleresidu}
  \end{equation}
This concludes the proof the lemma.\qed
\end{proof}
%%
%%
%%=========================================================================================================
%%
%%
Lastly, the final ingredient of the proof of theorem~\ref{theorem2} is
lemma~\ref{resistance_moyenne} that shows that the population expected effective
resistance, (\ref{expectedResistance}), is the effective resistance of the expected
adjacency matrix of $(\cS,\pr)$,
\begin{equation}
  \E{\bR} = \bR\big[\E{\bA}\big].
\end{equation}
%%
%%
%%=========================================================================================================
%%
%%
\begin{lemma}
  \label{resistance_moyenne}
  Let $\bR$ be the $n\times n$ effective resistance matrix of a network with adjacency matrix $\bA$. If
  \begin{equation}
    \bR = \displaystyle \frac{4}{n(p+q)} \bJ + \frac{p-q}{p+q} \frac{4}{n^2q} \bK,
    \mspace{24mu} \text{where} \mspace{8mu}
    \bK
    \mspace{8mu} \text{is given by (\ref{lamatriceK})}.
  \end{equation}
  Then $\bA = \bP$, where $\bP$ is given by
  (\ref{les_probabilites_de_succes}).
\end{lemma}
%%
%%
%%=========================================================================================================
%%
%%
\begin{proof}[Proof of lemma~\ref{resistance_moyenne}]
  The proof is elementary and relies on the following three identities. First, we recover
  $\bL^\dagger$, the pseudo-inverse of the combinatorial Laplacian $\bL = \bD - \bA$,
  from $\bR$,
  \begin{equation}
    \bL^\dagger = -\frac{1}{2} \bigg[\bI -\frac{1}{n}\bJ \bigg]\bR \bigg[\bI -\frac{1}{n}\bJ\bigg].
  \end{equation}
  We can then recover $\bL$ from $\bL^\dagger$; for every $\alpha \neq 0$, we have
  \begin{equation}
    \bL = \bigg[\bL^\dagger + \frac{\alpha}{n}\bJ \bigg]^{-1} - \frac{\alpha}{n}\bJ.
  \end{equation}
  Finally,  $\bA = - \bL + \diag(\bL)$. \qed
\end{proof}
This concludes the proof of theorem~\ref{theorem2}.
\end{proof}
%%
%%
%%=========================================================================================================
%%
%%
\section{Discussion of our results \label{discussion}}
%%
%%
%%=========================================================================================================
%%
%%
This paper provides analytical estimates of the sample \fr mean network when the sample is
generated from a stochastic block model. We derived the expression of the sample \fr mean
when the probability space $\SBM$ is equipped with two very different distances: the
Hamming distance and the resistance distance.  This work answers the question raised by
\cite{lunagomez21} ``what is the ``mean'' network (rather than how do we estimate the
success-probabilities of an inhomogeneous random network), and do we want the ``mean''
itself to be a network?''.

We show that the sample mean network is an unweighted network whose topology is usually
very different from the average topology of the sample. Specifically, in the regime of
networks where \mbox{$\min p_{ij} < 1/2$} (e.g., networks with $\o{n^2}$ but $\omega(n)$
edges), then the sample \fr mean is the empty network, and is pointless.  In contrast, the
resistance distance leads to a sample \fr mean that recovers the correct topology induced
by the community structure; the edge density of the sample \fr mean network is the
expected edge density of the random network ensemble. The effective resistance distance is
thus able to capture the large scale (community structure) and the mesoscale, which spans
scales from the global to the local scales (the degree of a vertex).

This work is significant because it provides for the first time analytical estimates of
the sample \fr mean for the stochastic blockmodel, which is at the cutting edge of
rigorous probabilistic analysis of random networks \citep{abbe18}. The technique of proof
that is used to compute the sample \fr mean for the Hamming distance can be extended to
the large class of inhomogeneous random networks \citep{kovalenko71}. It should also be
possible to extend our computation of the \fr mean with the resistance distance to
stochastic block models with $K$ communities of arbitrary size, and varying edge density.

From a practical standpoint, our work informs the choice of distance in the context where
the sample \fr mean network has been used to characterize the topology of networks for
network-valued machine learning (e.g., detecting change points in sequences of networks
\citep{ghoshdastidar20,zambon19}, computing \fr regression \citep{petersen19}, or cluster
network datasets \citep{xu20}). Future work includes the analysis of the sample \fr mean
when the distance is based on the eigenvalues of the normalized Laplacian \cite{wills20c}.

\section*{Acknowledgments}
F.G.M was supported by the National Natural Science
Foundation (\href{https://www.nsf.gov/awardsearch/showAward?AWD_ID=1815971}{CCF/CIF
  1815971}).
%% \bibliography{/Users/francois/LaTeX/Bib/biblio}

%%
%%
%%=========================================================================================================
%%
%%
\section{Additional proofs
\label{lespreuves}}
%%
%%
%%=========================================================================================================
%%
%%
%%
%%
%%=========================================================================================================
%%
%%
\subsection{Proof of lemma~\ref{lemma4}}
%%
%%
%%=========================================================================================================
%%
%%
We start with a simple result that provides an expression for the Hamming distance
squared. Let $\bA,\bB \in \cS$, and  let $\edgB$ denote the set of edges of $\bB$, and $\edgNB$ denote
the set of ``nonedges'' of $\bB$.  We denote by $\nEdgB$ the number of edges in
$\bB$. Then, the Hamming distance squared is given by
  \begin{equation}
    d^2_H(\bA,\bB) =  \nEdgB^2
    + 2 \nEdgB
    \Big[\MM
      \sum_{[i,j] \in \edgNB}\MM \aij  - \MM
      \sum_{[\ijp] \in \edgB} \MM \aijp
      \Big] 
    - \mspace{4mu} 4  \MMm \sum_{\EEbarre} 
    \MM \aij \aijp
    + \Big[\MM \sum_{\allij}\MM \aij \Big]^2 \label{F2ofB}
  \end{equation}
The proof of (\ref{F2ofB}) is elementary, and is omitted for brevity. We now provide the proof of lemma~\ref{lemma4}.
%%
%%
%%=========================================================================================================
%%
%%
%%
%%
%%=========================================================================================================
%%
%%
\begin{proof}[Proof of lemma~\ref{lemma4}]
Applying (\ref{F2ofB})  for each network $\Gk$, we get
  \begin{align}
    \F[2]{\bB}
    = &  \nEdgB^2
    + 2  \nEdgB
    \bigg[\sum_{[i,j] \in \edgNB} \frac{1}{N} \sum_{k=1}^N \akij
      -
      \MM \sum_{[\ijp] \in \edgB}  \frac{1}{N} \sum_{k=1}^N \akijp
      \bigg]\notag \\
    &  -  4 \MM \sum_{\EEbarre} \bigg[\frac{1}{N} \sum_{k=1}^N \akij \akijp \bigg]
    + \frac{1}{N} \sum_{k=1}^N \Big[\sum_{1 \leq i < j \leq n}\akij \Big]^2  
  \end{align}
  Using the expressions for the sample mean (\ref{sample_probability}) and correlation
  (\ref{sample_correlation}), and observing that
  \begin{equation}
    \frac{1}{N} \sum_{k=1}^N \Big[\MM \sum_{1 \leq i < j \leq n}\akij \Big]^2
    =   \MM \sum_{\iipjjp}
    \frac{1}{N} \sum_{k=1}^N \akij \akijp 
    =   \MM  \sum_{\iipjjp}  \MM \sE{\roij} \label{sommecarres},
  \end{equation}
  we get
  \begin{align}
    \F[2]{\bB}
    = &  \nEdgB^2
    +
    2  \nEdgB \Big [\MM \sum_{[i,j] \in \edgNB} \MM \hP_{ij}  - \mm \MM \sum_{[\ijp] \in \edgB} \MM \hP_{\ijp} \Big] 
    -4 \MM \mm \sum_{\EEbarre} \MM   \sE{\roij} 
    +  \MM \sum_{\iipjjp}  \MM \sE{\roij}. \label{Thenwehave}
  \end{align}
  Also, we have
  \begin{align}
    \nEdgB^2
     + 2  \nEdgB  & \Big[\mm \sum_{[i,j] \in \edgNB} \MM \hP_{ij}  - \MM \sum_{[\ijp] \in
        \edgB}  \MM \hP_{\ijp} \Big] \notag \\
     =  \bigg[ \nEdgB & - 2 \MM \sum_{(\ijp) \in \cE(\bB)} \MM \hP_{\ijp} \bigg]
    \bigg[\nEdgB + 2  \MM \sum_{(i,j) \in \edgNB} \MM \hP_{ij}\bigg] 
    + \pp 4 \MMm \sum_{\EEbarre} \hP_{ij}\hP_{\ijp}. \notag
  \end{align}
Whence
  \begin{align}
    \nEdgB^2
     + 2  \nEdgB  & \Big[\mm \sum_{[i,j] \in \edgNB} \MM \hP_{ij}  - \MM \sum_{[\ijp] \in
        \edgB}  \MM \hP_{\ijp} \Big] \notag \\
     =  \bigg[ \nEdgB & - 2 \MM \sum_{(\ijp) \in \cE(\bB)} \MM \hP_{\ijp} \bigg]
    \bigg[\nEdgB - 2 \MM \sum_{(\ijp) \in \cE(\bB)}  \MM \hP_{\ijp} + 2 \MM \sum_{1 \le i<j \le n}  \MM  \hP_{ij}  \bigg]
    +  4 \MM \sum_{\EEbarre} \MMm  \hP_{ij}\hP_{\ijp} \notag\\
    =  \bigg[ \nEdgB & - 2 \MM \sum_{(\ijp) \in \cE(\bB)} \MM \hP_{\ijp} \bigg]^2 + 2 \MM \sum_{\allij}  \MM  \hP_{ij}
    \bigg[\nEdgB - 2 \MM \sum_{(\ijp) \in \cE(\bB)}  \MM  \hP_{\ijp} \bigg]
    +  \Big[ \MM \sum_{\allij}  \MM \hP_{ij}\Big]^2  \notag\\
     +  4 \MM \sum_{\EEbarre} \MMm  & \hP_{ij}\hP_{\ijp} - \Big[ \MM \sum_{\allij}  \MM
       \hP_{ij}\Big]^2 \notag
  \end{align}
  Completing the square yields
  \begin{align}
  \nEdgB^2  + 2  \nEdgB   & \Big[\mm \sum_{[i,j] \in \edgNB} \mm \hP_{ij}  - \MM \sum_{[\ijp] \in
        \edgB}  \MM \hP_{\ijp} \Big] \notag \\
    =  \bigg[ \nEdgB & - 2 \MM \sum_{(\ijp) \in \cE(\bB)} \MM \hP_{\ijp} +   \sum_{\allij}  \MM  \hP_{ij} \Big]^2
    +  4 \MM \sum_{\stackrel{\scriptstyle [i,j] \in \edgNB}{\scriptsize (\ijp) \in
        \cE(\bB)}} \MMm  \hP_{ij}\hP_{\ijp} - \Big[ \MM \sum_{\allij}  \MM \hP_{ij}\Big]^2 \notag\\
    =  \Big[ \MM \sum_{(\ijp) \in \cE(\bB)}  \MM & (1 - 2 \hP_{\ijp} ) + \MM  \sum_{\allij}  \MM  \hP_{ij} \Big]^2
    +  4 \MM \sum_{\EEbarre}  \MMm  \hP_{ij}\hP_{\ijp} - \MMm \sum_{\iipjjp}  \MM \hP_{ij}\hP_{\ijp}.  \label{lecarre}
  \end{align}
  We can then substitute (\ref{sommecarres}) and (\ref{lecarre}) into (\ref{Thenwehave}),
  and we get the result advertised in the lemma,
  \begin{equation}
    \mm \F[2]{\bB}   =   \Big[\MM
      \sum_{[i,j] \in \edgB}  \MM \big (1 - 2\hP_{ij}\big)
      + \MMm \sum_{\allij} \MM \hP_{ij} \Big]^2
    \mm - \MMm \sum_{\iipjjp} \MMm \hP_{ij}  \hP_{\ijp}\hP_{\ijp}\p  - \sE{\roij}
     + 4 \MMm \sum_{\EEbarre}  \MMm  \hP_{ij} \hP_{\ijp}\p - \sE{\roij}, \label{F2_sample} 
  \end{equation}
  where we recognize the first term as $\delta^2(\bB,\hP)$. \qed
  \end{proof}
%%
%%
%%=========================================================================================================
%%
%%
\subsection{Proof of lemma~\ref{epsilon-concentrates}
\label{proof-epsilon-concentrates}}
%%
%%
%%=========================================================================================================
%%
%%
%%
%%
%%=========================================================================================================
%%
%%
\begin{proof}[Proof of lemma~\ref{epsilon-concentrates}]
  %%
%%
%%=========================================================================================================
%%
%%  
  We recall that the residual term $\zeta_N(\bB)$ is a sum of two types of terms,
  \begin{equation}
    \zeta_N(\bB) = \MM  \sum_{\EEbarre} \MM
    \hP_{ij} \hP_{\ijp} - \sE{\roij}.
    \label{le_residue}
  \end{equation}
  The sample mean $\hP_{ij}$, (\ref{sample_probability}), is the sum of $N$ independent
  Bernoulli random variables, and it concentrates around its mean $\pij$. The variation of
  $\hP_{ij}$ around $\pij$ is bounded by Hoeffding inequality,
  \begin{equation}
    \forall \pp 1 \le i< j \le n, \PP \forall N \ge 1, \quad
    \prob{\bAk \sim \SBM; \Big \lvert  \hP_{ij} -p_{ij} \Big \rvert  \ge \delta
    }
    \le
    \exp{\displaystyle \left(-2N \delta^2\right)}.
  \end{equation}
Let $\eps > 0$, and let $\alpha \eqdef \sqrt{\log{(n/\sqrt{\eps/2}})}$, a union bound yields
  \begin{equation}
     \forall N \ge 1, \PP \prob{\bAk \sim \SBM;
      \forall \pp 1 \le i < j < n,\quad
      \Big \lvert \hP_{ij} -\pij \Big \rvert \le \frac{\alpha}{\sqrt{N}}
     } > 1 -\eps/4.
     \label{hoeffding_union}
  \end{equation}
  The sample correlation, $\hRho_{\ijijp}$, (\ref{sample_correlation}), is evaluated in
  (\ref{le_residue}) for $[i,j] \in \edgB$ and $[\ijp] \in \edgNB$.  In this case, the
  edges $[i,j]$ and $[\ijp]$ are always distinct, thus $\akij$ and $\akijp$ are
  independent, and $\akij \akijp$ is a Bernoulli random variable with parameter $\pij
  \pijp$. We conclude that $\hRho_{\ijijp}$ is the sum of $N$ independent Bernoulli random
  variables, and thus concentrates around its mean, $\pij\pijp$.
  
  Let $\eps > 0$, and let $\beta \eqdef \sqrt{\log{(n^2/\sqrt{2\eps})}}$, Hoeffding
  inequality and a union bound yield
  \begin{equation}
    \forall N \ge 1, \quad 
    \prob{
      1 \le i< j \le n, 
      1 \le i^\prime< j^\prime\le n,\;  [i,j] \neq [\ijp],
      \bigg \lvert \sE{\roij}   -\pij \pijp \bigg \rvert \le \frac{\beta}{\sqrt{N}}
    }  > 1 - \eps/2,
    \label{hoeffding3}
  \end{equation}
  Combining (\ref{hoeffding_union}) and (\ref{hoeffding3}) yields
  \begin{align}
    \forall \eps > 0, & \pp \exists \alpha_1,\alpha_2, \beta, \PP \forall N \ge 1,
    \pp \forall \pp 1 \le  i  < j \le n, \pp  \forall \p 1 \le i^\prime < j^\prime\le n,  
     \pp [i,j] \neq [\ijp], \notag\\ 
          & \Big \lvert \hP_{ij}   -  \pij  \Big \rvert  \le \frac{\alpha_1}{\sqrt{N}}, \quad
          \Big \lvert \hP_{\ijp}  -  \pijp \Big \rvert \le \frac{1}{\sqrt{N}}, \quad \text{and} \quad 
          \Big \lvert \sE{\roij}  -  \pij \pijp \Big \rvert \le \frac{\beta}{\sqrt{N}},
          \label{approximations}
  \end{align}
  with probability $ 1 - \eps$. Lastly, combining (\ref{approximations}) and
  (\ref{le_residue}), we get the advertised result, 
  \begin{align}
    \forall&   \eps > 0,    \exists \; c > 0, \forall \; N  \ge 1,\notag\\
    & \pr  \Big (\pp
    \Big \lvert  \MMm  \sum_{\EEbarre} \MMm \hP_{ij}  \hP_{\ijp}  -\pij\pijp   - \sE{\roij} + \pij\pijp \Big \rvert
    \le \frac{c}{\sqrt{N}}
    \Big )
    = \pr \Big (
    \big \lvert \zeta_N(\bB) \big \rvert \le \frac{c}{\sqrt{N}}
    \Big)
    > 1 - \eps. \label{neglect1}
  \end{align}
\qed
\end{proof}
%%
%%
%%=========================================================================================================
%%
%%  
\subsection{Proof of lemma~\ref{variance_inequality}
  \label{proof_variance_inequality}}
%%
%%
%%=========================================================================================================
%%
%%
We first provide some inequalities (the proof of which are omitted) that relate
$\delta$ to the matrix norm $\lun{\p}$.
\begin{lemma}
  \label{delta_inequality}
  Let $\bA,\bB$ and $\bC$ be weighted adjacency matrices, with  $a_{ij},b_{ij},c_{ij} \in [0,1]$. We have
  \begin{equation}
    \frac{1}{2} \lun{\bA -\bB}  \leq \delta (\bA,\bB),
    \mspace{24mu}
    \text{and}
    \mspace{24mu}
    \frac{1}{2} \lun{\bA -\bC}  \leq \delta (\bA,\bB) + \delta(\bB,\bC).
  \end{equation}

\end{lemma}
%%
%%
%%=========================================================================================================
%%
%%
\begin{proof}[Proof of lemma~\ref{variance_inequality}]
Let $\bB, \smd{\pr} \in \cS$. From the definition of $\widehat{F}$ (see (\ref{la_fonction_F})) we have
  \begin{equation}
    \F[]{\bB} - \F[]{\smd{\pr}} = \delta^2(\bB,\hP) - \delta^2 (\smd{\pr}, \hP) =
    \Big(\delta(\bB,\hP) - \delta (\smd{\pr}, \hP) \Big)\Big(\delta(\bB,\hP) + \delta (\smd{\pr}, \hP) \Big).
  \end{equation}
  Because of lemma~\ref{delta_inequality}, we have
  \begin{equation}
    \delta(\bB,\hP) + \delta (\smd{\pr}, \hP) \ge \lun{\bB - \smd{\pr}}. \label{delta_plus_delta}
  \end{equation}
  Also,
  \begin{align}
    \delta(\bB,\hP) - \delta (\smd{\pr}, \hP) & =
    \MM \sum_{\allij} \MM b_{ij} - \smd{\pr}_{ij} -2 \MM \sum_{\allij} \hat{p}_{ij} (b_{ij} - \smd{\pr}_{ij}) \notag\\
    & =  \MM \sum_{\allij} \MM ( 1- 2\hat{p}_{ij}) (b_{ij} - \smd{\pr}_{ij}).
  \end{align}
  The entries of $\smd{\pr}$ are equal to $1$ only along $\cE(\smd{\pr})$, and $0$ along
  $\nE(\smd{\pr})$. Therefore,
  \begin{equation}
    \delta(\bB,\hP) - \delta (\smd{\pr}, \hP) 
     =  \MM \sum_{[i,j]\in \nE(\smd{\pr})} \MM b_{ij}( 1- 2\hat{p}_{ij})  + \MMm
    \sum_{[i,j]\in\cE(\smd{\pr})} \MM (1 - b_{ij})(2\hat{p}_{ij} -1).  \label{delta_difference}
  \end{equation}
  Let $\eps > 0$, because of the concentration of $\hat{p}_{ij} = \hP_{ij}$ around $\pij$, $\exists N_0$, $\forall
  N \ge N_0$,
  \begin{equation}
    \prob{\allij,\quad \big \lvert \hP_{ij} - \pij \big \rvert < \eps/2} > 1 -\eps. 
  \end{equation}
  We recall that we assume that $\lvert \pij - 1/2 \rvert > \eta, \quad \allij$, and
  therefore we get that for all $ 0 < \eps < 2\eta$,
  \begin{equation}
    \prob{\allij,\quad \Big \lvert 2\hP_{ij} - 1 \Big \rvert > 2\eta - \eps} > 1-\eps. \label{away_from_one_half}
  \end{equation}
  Because $\smd{\pr}$ is constructed using the majority rule, we have
  \begin{equation}
    \Big \lvert 2\hP_{ij} - 1 \Big \rvert = 
    \begin{cases}
      2\hat{p}_{ij} - 1  & \quad \text{if}\quad [i,j] \in \cE(\smd{\pr}),\\
      1 - 2\hat{p}_{ij}  & \quad \text{if}\quad [i,j] \in \nE(\smd{\pr}).
    \end{cases}
  \end{equation}
  Substituting the expression of $\Big \lvert 2\hP_{ij} - 1 \Big \rvert$ in 
  (\ref{away_from_one_half}) yields the following lower bounds, with probability $1-\eps$,
  \begin{equation}
    \begin{cases}
      2\hat{p}_{ij} - 1  > 2 \eta - \eps & \quad \text{if}\quad [i,j] \in \cE(\smd{\pr}),\\
      1 - 2\hat{p}_{ij} > 2 \eta - \eps & \quad \text{if}\quad [i,j] \in \nE(\smd{\pr}).
    \end{cases}
    \label{lower_bound_median}
  \end{equation}
  Inserting the inequalities given by (\ref{lower_bound_median}) into
  (\ref{delta_difference}) gives the following lower bound that happens with probability $1 - \eps$,
  \begin{equation}
    \delta(\bB,\hP) - \delta (\smd{\pr}, \hP)
    \ge
    (2\eta - \eps) \Big[\MMm
      \sum_{[i,j]\in\cE(\smd{\pr})} \MM (1 - b_{ij}) + \MMm \sum_{[i,j]\in\nE(\smd{\pr}_{ij})} \MM b_{ij}
      \Big].
  \end{equation}
  We bring the proof to an end by observing that
  \begin{align}
    \lun{\bB - \smd{\pr}}  
    & = \MM \sum_{\allij} \MM \big \lvert b_{ij} - \smd{\pr}_{ij} \big \rvert 
    = \MMM \sum_{[i,j]\in\cE(\smd{\pr})} \MMm \big \lvert b_{ij} - 1 \big \rvert + \MMm
    \sum_{[i,j] \in \nE(\smd{\pr})} \MMm \big \lvert  b_{ij} \big \rvert
    = \MMm \sum_{[i,j]\i n\cE(\smd{\pr})} \MMm (1 - b_{ij})
    + \MMM \sum_{[i,j]\in \nE(\smd{\pr})} \MMm b_{ij},
  \end{align}
  whence we conclude that
  \begin{equation}
    \delta(\bB,\hP) - \delta (\smd{\pr}, \hP)
    \ge
    (2\eta - \eps) \lun{\bB - \smd{\pr}}, \label{delta_moins_delta}
  \end{equation}
  with probability $1-\eps$. Finally, combining (\ref{delta_plus_delta}) and
  (\ref{delta_moins_delta}), and letting $\alpha \eqdef 2\eta -\eps$, we get the
  inequality advertised in lemma~\ref{variance_inequality},
  \begin{equation}
    \alpha \lun{\bB - \smd{\pr}}^2 \le \big \lvert \F[]{\bB} - \F[]{\smd{\pr}} \big \rvert,
  \end{equation}
  with probability $1-\eps$.
\qed
\end{proof}
%%
%%
%%=========================================================================================================
%%
%%
\subsection{Proof of equation~(\ref{tailleresidu})
\label{preuve_tailleresidu}}
%%
%%
%%=========================================================================================================
%%
%%
We show
\begin{equation}
  \big \lvert
  \lbr \bu_i - \bu_j , \bQ (\bu_i - \bu_j)\rbr    
  \big \rvert
  \leq \left (\frac{1}{d_i} + \frac{1}{d_j}\right) \frac{8\sqrt{2}}{(np)^{3/2}} \quad \text{almost surely.}
    \label{tailleresidu2}
\end{equation}

%%
%%
%%=========================================================================================================
%%
%%
\begin{proof}
  Let $i,j\in [n]$. We have
  \begin{align}
    \bigg \lvert \lbr \bu_i - \bu_j, \bQ(\bu_i-\bu_j)\rbr \bigg \rvert
    & = 
    \bigg \lvert  \lbr \bu_i - \bu_j,
    \sum_{m=3}^n \frac{\lambda_m}{ 1- \lambda_m} \bPi_m 
    (\bu_i-\bu_j)\rbr \bigg \rvert \\
    & \le
    \bigg \| \sum_{m=3}^n \frac{\lambda_m}{ 1- \lambda_m} \bPi_m \bigg \|\| \bu_i-\bu_j\|^2
     \le
    \bigg\{\frac{1}{d_i} + \frac{1}{d_j} \bigg\}
    \bigg \| \sum_{m=3}^n \frac{\lambda_m}{ 1- \lambda_m} \bPi_m \bigg \|.
      \label{threeterms}
  \end{align}
  Now
  \begin{equation}
    \bigg \| \sum_{m=3}^n \frac{\lambda_m}{ 1- \lambda_m} \bPi_m \bigg \|^2 =
    \sum_{m=3}^n \bigg\lvert \frac{\lambda_m}{ 1- \lambda_m}
    \bigg\rvert^2 \big \| \bPi_m \big \|^2
    \le 
    \sum_{m=3}^n 2 \big\lvert \lambda_m  \big\rvert^2 \big \| \bPi_m \big \|^2
    \le
    2 \bigg[\max_{m=2}^n \big\lvert \lambda_m  \big\rvert \bigg]^2,
  \end{equation}
  because the $\bPi_m$ are orthonormal projectors such that $\sum_{m=1}^n \bPi_m = \bI$.
  Using the  following concentration result (e.g., Theorem 3.6 in \citep{Chung2003a}),
  \begin{equation}
    \max_{m=3}^n \big \lvert \lambda_m \big \rvert \leq \frac{8}{\sqrt{np}} \quad \text{almost surely.}
    \label{lambdamax}
  \end{equation}
  we conclude that
  \begin{equation}
    \big \lvert \lbr \bu_i - \bu_j, \bQ (\bu_i - \bu_j)\rbr \big \rvert 
    \le
    \left(\frac{1}{d_i} + \frac{1}{d_j} \right) \frac{8\sqrt{2}}{\left(np\right)^{3/2}} \quad \text{almost surely.}
    \label{order3}
  \end{equation}
\qed
\end{proof}
\end{document}